\newif\ifreport\reporttrue
\documentclass[conference,letterpaper]{IEEEtran}
\IEEEoverridecommandlockouts

\usepackage{setspace}
\usepackage{cite}
\usepackage{graphicx}
\usepackage{caption}
\usepackage{bbm}
\usepackage{subcaption}
\usepackage[cmex10]{amsmath}
\usepackage{mathtools}
\usepackage{amsthm}
\usepackage{amsfonts}
\usepackage{amssymb}
\usepackage{mathdots}
\usepackage{mathtools}
\usepackage{enumitem}
\usepackage{algorithm}
\usepackage{algorithmic}
\usepackage{bm,xstring}
\usepackage{fixltx2e}
\usepackage{tikz}
\usetikzlibrary{decorations.pathreplacing}
\usepackage{color}
\usepackage{tcolorbox}
\colorlet{blue}{black}
\newtheorem{theorem}{Theorem}
\newtheorem{lemma}{Lemma}

\newtheorem{definition}{Definition}

\newcommand{\ignore}[1]{{}}

\newcommand{\yangc}[1]{\textcolor{red}{\textbf{Lei: #1}}}
\newcommand{\yan}[1]{\ifbool{inccomment}{{\color{blue}#1}}{}}
\newcommand{\yanc}[1]{\textcolor{blue}{\textbf{Feng: #1}}}

\begin{document}
\title{The Age of Correlated Features in Supervised Learning based Forecasting}
\author{
\IEEEauthorblockN{Md Kamran Chowdhury Shisher\IEEEauthorrefmark{1}, Heyang Qin\IEEEauthorrefmark{2}, Lei Yang\IEEEauthorrefmark{2}, Feng Yan\IEEEauthorrefmark{2}, \text{and} Yin Sun\IEEEauthorrefmark{1}}
\IEEEauthorblockA{\IEEEauthorrefmark{1}Department of ECE, Auburn University, AL, USA}
\IEEEauthorblockA{ \IEEEauthorrefmark{2}Department of CSE, University of Nevada, Reno, NV, USA}
\thanks{This work was supported in part by NSF grants CCF-1813050, IIS-1838024, EEC-1801727, CNS-1950485, CCF-1756013, and ONR grant N00014-17-1-2417.}
}


\maketitle
\begin{abstract}
In this paper, we analyze the impact of information freshness on supervised learning based forecasting. In these applications, a neural network is trained to predict a time-varying target (e.g., solar power), based on multiple correlated features (e.g., temperature, humidity, and cloud coverage). The features are collected from different data sources and are subject to heterogeneous and time-varying ages. By using an information-theoretic approach, we prove that the minimum training loss is a function of the ages of the features, where the function is not always monotonic. However, if the empirical distribution of the training data is close to the distribution of a Markov chain, then the training loss is approximately a non-decreasing age function. Both the training loss and testing loss depict similar growth patterns as the age increases. An experiment on solar power prediction is conducted to validate our theory. Our theoretical and experimental results suggest that it is beneficial to (i) combine the training data with different age values into a large training dataset and jointly train the forecasting decisions for these age values, and (ii) feed the age value as a part of the input feature to the neural network. 
\end{abstract}

\ignore{\yanc{it might be useful to have the structure and logic of intro. Below is an initial draft of the story-line, please feel free to improve it. \\ 
Importance and general background of supervised learning. \\
Importance and general background of AoI in supervised learning.   \\
Literature review of single feature AoI for supervised learning.  \\
Why single feature AoI is not sufficient? \\
Who correlated features AoI is challenge? \\
Our key insights and solutions. \\
Use temp and solar applications for case study and the preliminary results.  \\
Potential future work and the profound influence for supervised learning. \\
}}

\section{Introduction}

Recently, the proliferation of artificial intelligence and cyber physical systems has engendered a significant growth in machine learning techniques for time-series forecasting applications, such as autonomous driving \cite{ pedestrian_intent, autonomous_driving}, energy forecasting  \cite{energy_forecast, kaur2020energy, solar}, and traffic prediction \cite{traffic_forecast}. In these applications, a predictor (e.g., a neural network) is used to infer the status of a time-varying target (e.g., solar power) based on several features (e.g., temperature, humidity, and cloud coverage). Fresh features are  desired, because they could potentially lead to a better forecasting performance. For example, recent studies on pedestrian intent prediction \cite{pedestrian_intent} and autonomous driving \cite{autonomous_driving} showed that prediction accuracy can be greatly improved if fresher data is used. Similarly, it was found in \cite{kaur2020energy, solar} that the performance of energy forecasting degrades as the observed feature becomes stale. This phenomenon has also been observed in other applications of time-series forecasting, such as traffic control \cite{traffic1} and financial trading \cite{stock}.

\emph{Age of information (AoI)}, or simply \emph{age}, is a performance metric that measures the freshness of the information that a receiver has about the status of a remote source \cite{kaul2012real}. Recent research efforts on AoI have been focused on analyzing and optimizing the AoI in, e.g., communication networks \cite{sun2017, kaul2011minimizing, sun2019, zhou2018optimal, arafa2019timely,yates2015lazy, energy_harvesting}, remote estimation \cite{sun2017remote, orneesampling}, and control systems \cite{klugel2019aoi, BarasControl}. However, the impact of AoI on supervised learning based forecasting has not been well-understood, despite its significance in a broad range of applications. Recently, an \emph{age of features} concept was studied in \cite{Kamran}, where a stream of features are collected progressively from a single data source and, at any time, only the freshest feature is used for prediction. Meanwhile, in many applications, the forecasting decision is made by jointly using multiple features that are collected in real-time from different data sources (e.g., temperature readings from  thermometers and wind speed measured from  anemometers). These features are of diverse measurement frequency, data formats, and may be received through separated communication channels. Hence, their AoI values are different. This motivated us to analyze the performance of supervised learning algorithms for time-series forecasting, where the features are subject to heterogeneous and time-varying ages. The main contributions of this paper are summarized as follows: 
\begin{figure}
\centering
\includegraphics[width=90mm]{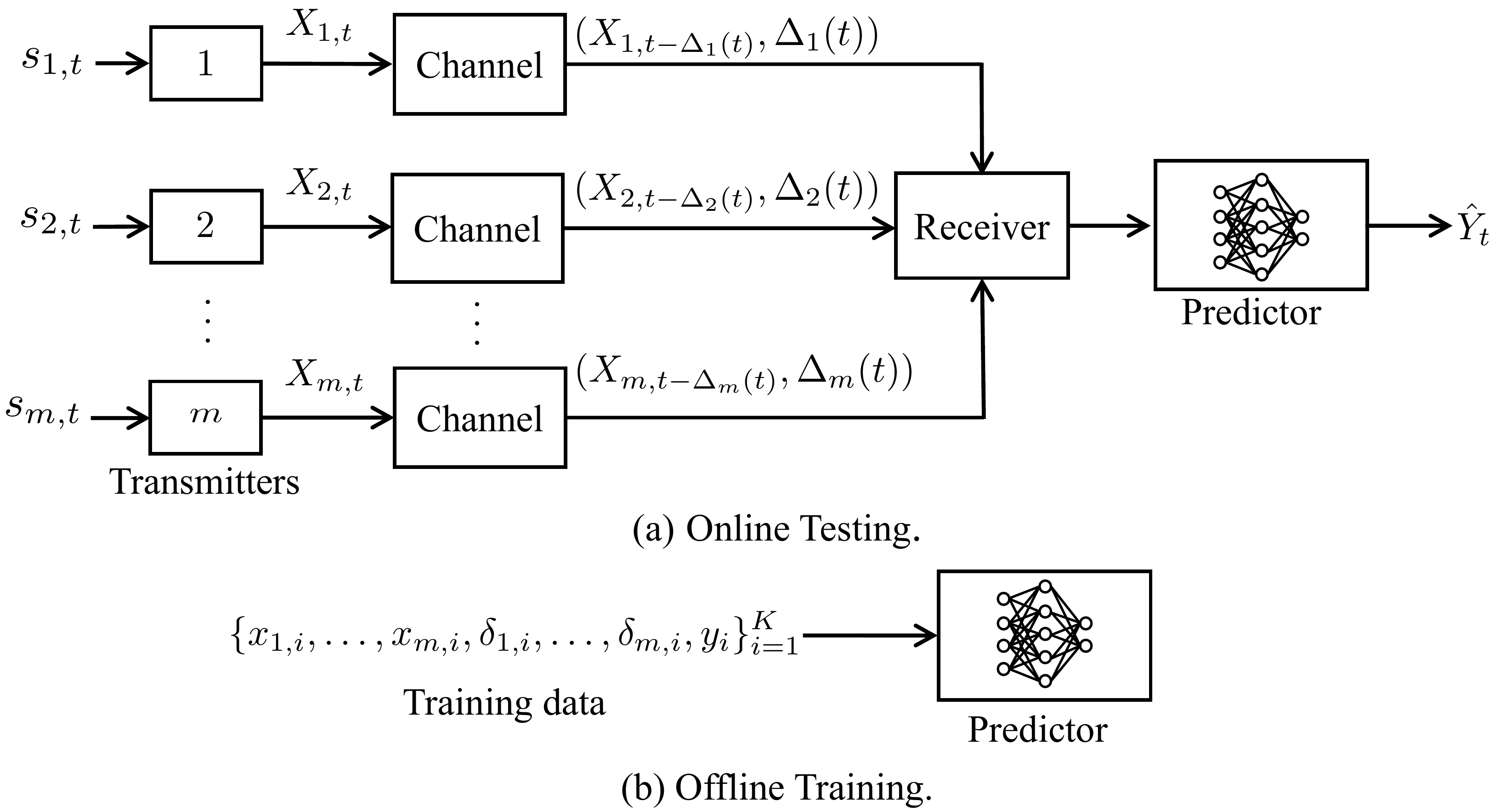}
\caption{System Model.}
\label{System_model}
\end{figure}

\begin{itemize}[noitemsep,leftmargin=*]
    \item We present an information theoretic approach to interpret the influence of age on supervised learning based forecasting. Our  analysis shows that the minimum training loss is a multi-dimensional function of the age vector of the features, but the function is not necessarily monotonic. This is a key difference from the non-decreasing age metrics considered in earlier work, e.g., \textcolor{blue}{\cite{sun2019, kosta2017age, klugel2019aoi} and the references therein.} 
    
    \item Moreover, by using a local information geometric analysis, we prove that if the empirical distribution of training data samples can be accurately approximated as a Markov chain, then the minimum training loss is close to a non-decreasing function of the age. The testing loss performance is analyzed in a similar way. 
  
    \item We compare the performance of several training approaches and find that it is better to (i) combine the training data with different age values into  a  large  training  dataset and jointly train the forecasting decisions for these age values, and (ii) add the age value as a part of the feature. This training approach has a lower computational complexity than the separated training approach used in \cite{time-series}, where the forecasting decision for each age value is trained individually. Experimental results on solar power prediction are provided to validate our  findings. 
\end{itemize}

\section{System Model}
Consider the learning-based time-series forecasting system in Fig. \ref{System_model},  which consists of $m$ transmitters and one receiver. \textcolor{blue}{The system time is slotted. Each transmitter} $l$ takes measurements from a discrete-time signal process $s_{l,t}$. The processes \textcolor{blue}{$s_{1,t},\ldots, s_{m,t}$} contain useful information for inferring the behavior of a target process $Y_t$. Transmitter $l$ progressively generates a sequence of features  $X_{l,t}$ from the process $s_{l,t}$. Each feature $X_{l,t}=f(s_{l,t-\tau}, s_{l,t-1-\tau},\ldots,$ $ s_{l,t-b+1-\tau})$ is a function of a finite-length time sequence from the process $s_{l,t}$, where $b$ is the length of the sequence and $\tau$ is the processing time needed for creating the feature. The processes \textcolor{blue}{$s_{1,t},\ldots, s_{m,t}$} may be correlated with each other, and so are the features \textcolor{blue}{$X_{1,t},\ldots, X_{m,t}$.} The features are sent from the transmitters to the receiver through one or multiple channels. The receiver \textcolor{blue}{feeds the features to a predictor (e.g., a trained neural network), which infers the current target value $Y_t$.}



Due to transmission errors and random transmission time, freshly generated features may not be immediately delivered to the receiver. Let $G_{l,i}$ and $D_{l,i}$ be the creation time and delivered time of the $i$-th feature of the process $s_{l, t}$, respectively, \textcolor{blue}{such that} $G_{l,i} \leq G_{l, i+1}$ and $G_{l,i} \leq D_{l, i}$. Then, $U_l(t)=\max \{G_{l,i}: D_{l,i} \leq t\}$ is the creation time of the freshest feature that was \textcolor{blue}{generated} from process $s_{l,t}$ and has been delivered to the receiver by time $t$. At time $t$, the receiver uses the $m$ freshest delivered features $(X_{1, U_1(t)},\ldots, X_{m, U_m(t)})$, each from a transmitter, to predict the current target $Y_t$. The age of the features generated from process $s_{l,t}$ is defined as
\begin{align}
\Delta_l(t) = t- U_l(t)=t-\max \{G_{l,i}: D_{l,i} \leq t\},
\end{align} 
which is the time elapsed since the creation time $U_l(t)$ of the freshest delivered feature $X_{l, U_l(t)}$ up to the current time $t$. If $\Delta_l(t)$ is small, then there exists a fresh delivered feature that was created recently from process $s_{l,t}$. The evolution of $\Delta_l(t)$ over time is illustrated in Fig. \ref{fig:age1}. 

The predictor is trained by \textcolor{blue}{using} an Empirical Risk Minimization (ERM) based supervised learning algorithm, such as logistic regression and \textcolor{blue}{linear regression}. A supervised learning algorithm consists of two phases: \emph{offline training} and \emph{online testing}. \textcolor{blue}{In offline training phase, the predictor is trained by minimizing an expected loss function under the empirical distribution of a training dataset.} Each entry $(x_{1,i},\ldots,x_{m,i},\delta_{1,i},\ldots,\delta_{m,i}, y_i)$ of the training dataset contains  $m$ features $(x_{1,i},$ $\ldots,x_{m,i})$, the age values of the features $(\delta_{1,i},\ldots,\delta_{m,i})$, and the target $y_i$. In Section \ref{sec_interpretation}, we will see that it is important to add the age values into training data. \ignore{\textcolor{blue}{As the standard approach of supervised learning, ERM selects the predictor which provides minimum empirical risk, where the empirical risk is the expected prediction error on the training data set.}}
In online testing, the trained predictor is used to predict the target in real-time, as explained above.

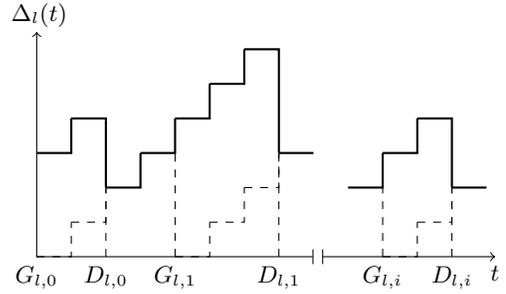
\begin{figure}
\centering
\begin{tikzpicture}[scale=0.23]
\draw [<-|] (0,13)  -- (0,0) -- (16,0);
\draw [|->] (16.5,0) -- (26.5,0) node [below] {\small$t$};
\draw (-2,14) node [right] {\small$\Delta_l(t)$};
\draw (0,0) node [below] {\small$G_{l,0}$};
\draw (4,0) node [below] {\small$D_{l,0}$};
\draw (8,0) node [below] {\small$G_{l,1}$};
\draw (14,0) node [below] {\small$D_{l,1}$};
\draw (20,0) node [below] {\small$G_{l,i}$};
\draw (24,0) node [below] {\small$D_{l,i}$};
\draw[ thick, domain=0:2] plot (\x, {6})  -- (2, {8});
 \draw [ thick, domain=2:4] plot (\x, {8})  -- (4, {4});
 \draw[ thick, domain=4:6] plot (\x, {4})  -- (6, {6});
  \draw[ thick, domain=6:8] plot (\x, {6})  -- (8, {8});
  \draw[ thick, domain=8:10] plot (\x, {8})  -- (10, {10});
  \draw[ thick, domain=10:12] plot (\x, {10})  -- (12, {12});
    \draw[ thick, domain=12:14] plot (\x, {12})  -- (14, {6});
    \draw[ thick, domain=14:16] plot (\x, {6}) ;
      \draw[ thick, domain=18:20] plot (\x, {4})  -- (20, {6});
      \draw[ thick, domain=20:22] plot (\x, {6})  -- (22, {8});
      \draw[ thick, domain=22:24] plot (\x, {8})  -- (24, {4});
      \draw[ thick, domain=24:26] plot (\x, {4}) ;
\draw[  thin,dashed,  domain=4:4] plot (\x, {0})-- (4, 4);
\draw[  thin,dashed,  domain=8:8] plot (\x, {0})-- (8, 6);
\draw[  thin,dashed,  domain=14:14] plot (\x, {0})-- (14, 6);
\draw[  thin,dashed,  domain=20:20] plot (\x, {0})-- (20, 4);
\draw[  thin,dashed,  domain=24:24] plot (\x, {0})-- (24, 4);

\draw[ thin,dashed, domain=0:2] plot (\x, {0})  -- (2, {2});
 \draw [ thin,dashed, domain=2:4] plot (\x, {2})  -- (4, {4});
 \draw[ thin,dashed, domain=8:10] plot (\x, {0})  -- (10, {2});
  \draw[ thin,dashed, domain=10:12] plot (\x, {2})  -- (12, {4});
    \draw[thin,dashed, domain=12:14] plot (\x, {4})  -- (14, {6});
    \draw[ thin,dashed, domain=20:22] plot (\x, {0})  -- (22, {2});
      \draw[ thin,dashed, domain=22:24] plot (\x, {2})  -- (24, {4});
\end{tikzpicture}
\caption{Evolution of the age $\Delta_l(t)$ in discrete-time.}
\label{fig:age1}
\vspace{-0.5cm}
\end{figure}

The goal of this paper is to interpret how the age processes $\Delta^m(t)=(\Delta_1(t),\ldots,\Delta_m(t))$ of the features affect the  performance of time-series forecasting.

\section{Performance of Supervised Learning From An Information Theoretic Perspective}\label{Information_Theoretic_interpretation}

In this section, we introduce 
several information theoretic measures that characterize the fundamental limits for the training and testing performance of supervised learning. Based on these information theoretic measures, the influence of information freshness on supervised learning based forecasting will be analyzed subsequently in Section \ref{sec_interpretation}.

Let $X^m = (X_1,\ldots, X_m)$ represent a vector of $m$ random features, which takes value $x^m = (x_1,\ldots, x_m)$ from the finite space $\mathcal X^m = \mathcal X_1\times \mathcal X_2 \times \ldots \times\mathcal X_m$. 
As the standard approach for supervised learning, ERM is a stochastic decision problem $(\mathcal X^m, \mathcal Y, \mathcal A, L)$, where a decision-maker predicts $Y \in \mathcal Y$ by taking an action $a=\psi(X^m) \in \mathcal A$ based on features $X^m \in \mathcal X^m$. The performance of ERM is measured by a loss function $L: \mathcal Y \times \mathcal A \mapsto \mathbb R$, where $L(y, a)$ is the incurred loss if action $a$ is chosen when $Y=y$. For example, $L$ is a logarithmic function $L_{\log}(y, P_Y) = -\log P_Y (y)$ in logistic regression and a quadratic function $L_2(y,\hat y) = (y -\hat y)^2$ in linear regression. 
Let $P_{X^m,Y}$ and $P_{\tilde X^m,\tilde Y}$, respectively, denote the empirical distributions of the training data and testing data, where $\tilde X^m$ and $\tilde Y$ are random variables with a joint distribution $P_{\tilde X^m,\tilde Y}$. We restrict our analysis in which the marginals $P_{X^m}$ and $P_{\tilde X^m}$ are strictly positive. 

\subsection{Minimum Training Loss}

The objective of training in ERM-based supervised learning is to solve the following problem:
\begin{align}\label{cond_entropy}
H_L(Y|X^m)= \min_{\psi \in \Psi} \mathbb E_{X^m, Y \sim P_{X^m, Y}}[L(Y,\psi(X^m))],
\end{align}
where $\Psi$ is the set of allowed decision functions and $H_L(Y|X^m)$ is the \emph{minimum training loss}. We consider a case that $\Psi$ contains \emph{all} functions from $\mathcal X^m$ to $\mathcal A$. Such a choice of $\Psi$ is of particular interest for two reasons: (i) Since $\Psi$ is quite large, $H_L(Y|X^m)$ provides a fundamental lower bound of the training loss for any ERM based learning algorithm. \textcolor{blue}{(ii) When $\Psi$ contains all functions from $\mathcal X^m$ to $\mathcal A$, Problem \eqref{cond_entropy} can be reformulated as 
\begin{align}\label{eq_decompose}
&H_L(Y|X^m) \nonumber\\
= & \min_{\substack{\psi(x^m)\in \mathcal A, \\
\forall~x^m\in \mathcal X^m} } \sum_{x^m \in \mathcal X^m} \!\! P_{X^m}(x^m) \mathbb{E}_{Y\sim P_{Y|X^m=x^m}}[L(Y,\psi(x^m))] \nonumber\\
=& \sum_{x^m \in \mathcal X^m} \!\! P_{X^m}(x^m) \min_{\psi(x^m)\in\mathcal A} \mathbb{E}_{Y\sim P_{Y|X^m=x^m}}[L(Y,\psi(x^m))], 
\end{align}
where, in the last step, the training problem is decomposed into a sequence of separated optimization problems, each optimizing an action $\psi(x^m)$ for given $x^m \in \mathcal X^m$.} For the considered $\Psi$, $H_L(Y| X^m)$ in (2) is termed the generalized conditional entropy of $Y$ given $X^m$ \cite{David_Tse}.
Similarly, the {generalized (unconditional) entropy} $H_L(Y)$ is defined as \cite{David_Tse, DawidC, Dawid}
\begin{align}\label{eq_entropy}
H_L(Y) = \min_{a\in\mathcal A} \mathbb{E}_{Y\sim P_{Y}}[L(Y,a)].
\end{align}
The optimal solution to \eqref{eq_entropy} is called \emph{Bayes action}, which is denoted as $a_{P_Y}$. If the Bayes actions are not unique, one can pick any such Bayes action as $a_{P_Y}$ \cite{dawid2014theory}. 
The generalized conditional entropy for $Y$ given $X^m=x^m$ is \cite{David_Tse}
\begin{align}\label{eq_cond_entropy}
H_L(Y| X^m=x^m)=&\min_{a\in\mathcal A} \mathbb{E}_{Y\sim P_{Y|X^m=x^m}}[L(Y,a)] \nonumber\\
=&\min_{\psi(x^m)\in\mathcal A} \mathbb{E}_{Y\sim P_{Y|X^m=x^m}}[L(Y,\psi(x^m))].
\end{align}
Substituting \eqref{eq_cond_entropy} into \eqref{eq_decompose}, yields the relationship
\begin{align}\label{eq_cond_entropy1}
H_L(Y| X^m)=\!\!\sum_{x^m \in \mathcal X^m} \!\! P_{X^m}(x^m) \ H_L(Y| X^m=x^m).
\end{align}
{\textcolor{blue}{We assume that entropy and conditional entropy discussed in this paper are bounded.}} We also need to define the generalized mutual information, which is given by
\begin{align}\label{MI}
I_L(Y; X^m)=H_L(Y)-H_L(Y|X^m).
\end{align}
Examples of the loss function $L$ and the associated generalized entropy $H_L(Y)$ were discussed in \cite{Kamran, David_Tse, Dawid, DawidC}. 

\subsection{Testing Loss}
The testing loss, also called validation loss, of supervised learning is the expected loss on testing data using the trained predictor. In the sequel, we use the concept of generalized cross entropy to characterize the \emph{testing loss}. The generalized cross entropy between $Y$ and $\tilde Y$ is defined as 
\begin{align}
H_L(\tilde Y; Y)= \mathbb{E}_{Y \sim P_{\tilde Y}}\left[L(Y, a_{P_Y})\right],
\end{align}
where $a_{P_Y}$ is the Bayes action defined above. 
Similarly, the generalized conditional cross entropy between $\tilde Y$ and $Y$ given $\tilde X^m=x^m$ is 
\begin{align}\label{cond_cross}
H_L(\tilde Y; Y|\tilde X^m=x^m)= \mathbb{E}_{Y \sim P_{\tilde Y|\tilde X^m=x^m}}\left[L(Y, a_{P_{Y|X^m=x^m}})\right],
\end{align}
where $a_{P_{Y|X^m=x^m}}$ is the Bayes action of a predictor that was trained by using the empirical conditional distribution $P_{Y|X^m=x^m}$ of the training data and $P_{\tilde Y|\tilde X^m=x^m}$ is the empirical conditional distribution of the testing data. 
Using \eqref{cond_cross}, the testing loss of supervised learning can be expressed as
\begin{align}\label{TestError}
\!\!\!\!\!H_L(\tilde Y; Y| \tilde X^m)=\!\!\!\sum_{x^m \in \mathcal X^m}\!\! \tilde P_{\tilde X^m}(x^m) H_L(\tilde Y; Y| \tilde X^m=x^m),\!\!\!\!\!
\end{align}
which is also termed the generalized conditional cross entropy between $\tilde Y $ and $Y$ given $\tilde X^m$.

\section{Impact of Information Freshness on Supervised Learning based Forecasting}\label{sec_interpretation}
In this section, we analyze the training and testing loss performance of supervised learning under different age values. It is shown that the minimum training loss is a function of the age, but the function is not always monotonic. By using a local information geometric approach, we provide a sufficient condition under which the training loss can be closely approximated as a non-decreasing age function. Similar conditions for the monotonicity of the testing loss on age are also discussed. 
\subsection{Training Loss under Constant AoI}\label{Constant_AoI}
\textcolor{blue}{For ease of understanding, we start by} analyzing the minimum training loss under a constant AoI, i.e, $\Delta^m(t)=\delta^m$, for all time $t$. The more practical case of time-varying AoI will be studied subsequently in Section \ref{Dynamic_AoI}.

Markov chain has been a widely-used model in \textcolor{blue}{time-series analysis} \cite{Solar2017, ProbabilisticMarkov, Stockmarkov}. Define $X_{t-\tau^m}^m = (X_{1,t-\tau_1}, \ldots, \textcolor{blue}{X_{m,t-\tau_m}})$ for $\tau^m = (\tau_1, \ldots, \tau_m)$, \textcolor{blue}{where $X_{l,t-\tau_l}$ is the feature generated from transmitter $l$ at time $t-\tau_l$}. If $Y_t \leftrightarrow X^m_{t-\tau^m} \leftrightarrow X^m_{t-\tau^m-\mu^m}$ is a Markov chain for all $\mu^m, \tau^m \geq 0$ (assume the Markov chain is also stationary), then by using the data processing inequality  \cite{DawidC}, one can show that the minimum training loss $H_L(Y_t|X^m_{t-\delta^m})$ is 
a non-decreasing function of the age vector $\delta^m$. 
\textcolor{blue}{However, practical time-series signals are rarely Markovian \cite{Non-Markovremarks, non-Markovnoise,selectiveMarkov, lstmNon-Markov} and, as a result, the minimum training loss $H_L(Y_t|X^m_{t-\delta^m})$ is not always a monotonic function of $\delta^m$.}

To develop a unified \textcolor{blue}{framework} for analyzing the minimum training loss $H_L(Y_t|X^m_{t-\delta^m})$, we consider a relaxation of the Markov chain model, called \emph{$\epsilon$-Markov chain}, which was proposed recently in \cite{Kamran}. 
\begin{definition}
$\epsilon$-\textbf{Markov Chain:} \cite{Kamran} Assume that the distributions $P_{Y|X}, P_{Z|X},$ and $P_X$ are strictly positive. Given $\epsilon \geq 0$, a sequence of random variables $Y, X,$ and $Z$ is said to be an $\epsilon$-Markov chain, denoted as $Y \xleftrightarrow{\epsilon} X \xleftrightarrow{\epsilon} Z$, \textcolor{blue}{if
\begin{align}\label{epsilon-Markov-def}
D_{\chi^2}\left(P_{Y,X,Z} || P_{Y|X} P_{Z|X} P_X \right) \leq \epsilon^2,
\end{align}
where
$D_{\chi^2}(P_Y||Q_Y)$ is Neyman's $\chi^2$-divergence, given by
\begin{align}\label{chi-divergence}
D_{\chi^2}(P_Y ||Q_Y) = \sum_{y \in \mathcal{Y}} \frac{(P_Y(y) - Q_Y(y))^2}{Q_Y(y)}.
\end{align}}
\end{definition}
\textcolor{blue}{The inequality  \eqref{epsilon-Markov-def} can be also expressed as 
\begin{align}
I_{\chi^2}(Y;Z|X) \leq \epsilon^2,
\end{align}
where $I_{\chi^2}(Y;Z|X)$ is the $\chi^2$-conditional mutual information.} If $\epsilon=0$, then $Y \xleftrightarrow{\epsilon} X \xleftrightarrow{\epsilon} Z$ reduces to a Markov chain. Hence, $\epsilon$-Markov chain is more general than Markov chain. 

\textcolor{blue}{For $\epsilon$-Markov chain, a relaxation of the data processing inequality is provided in the following lemma:
\begin{lemma}[$\epsilon$-data processing inequality] \cite{Kamran} \label{Lemma_CMI}
If $Y \xleftrightarrow{\epsilon} X \xleftrightarrow{\epsilon} Z$ is an $\epsilon$-Markov chain and $H_L(Y)$ for the loss function $L$ is twice differentiable in $P_Y$, then we have (as $\epsilon \rightarrow 0)$
\begin{align}
I_L(Y; Z | X) &=O(\epsilon^2),\\
H_L(Y|X) &\leq H_L(Y|Z)+O(\epsilon^2).
\end{align}
\end{lemma}}

By using Lemma \ref{Lemma_CMI}, the training loss performance of supervised learning based forecasting is characterized in the following theorem. For notational simplicity, the theorem is presented for the case of $m=2$ features, whereas it can be easily generalized to any positive integer value of $m$. 

\begin{theorem}\label{theorem1}
Let $\{(X_{1, t}, X_{2, t}, Y_t), t \geq 0\}$ be a stationary stochastic process. 
\begin{itemize}
\item[(a)] The minimum training loss $H_L(Y_t |X_{1, t-\delta_1}, X_{2, t-\delta_2})$ is a function of $\delta_1$ and $\delta_2$, determined by
\begin{align}\label{eMarkov}
\!\!\!\!\!\!H_L(Y_t |X_{1, t-\delta_1}, X_{2, t-\delta_2})= f_1(\delta_1, \delta_2)-f_2(\delta_1, \delta_2),\!\!
\end{align}
where $f_1(\delta_1, \delta_2) $ and $f_2(\delta_1, \delta_2)$ are two non-decreasing functions of $\delta_1$ and $\delta_2$, given by
\begin{align}\label{functionf_1}
&f_1(\delta_1, \delta_2)\! \nonumber\\
=&H_L(Y_t| X_{1,t}, X_{2,t}) \nonumber\\
& +\sum_{k=0}^{\delta_1-1} I_L(Y_t; X_{1, t-k}, X_{2, t-\delta_2} | X_{1, t-k-1}, X_{2, t-\delta_2}) \nonumber\\
& +\sum_{k=0}^{\delta_2-1} I_L(Y_t; X_{1, t}, X_{2, t-k} | X_{1, t}, X_{2, t-k-1}),\\
\label{functionf_2}
&f_2(\delta_1, \delta_2)\! \nonumber\\
=&\sum_{k=0}^{\delta_1-1} I_L(Y_t; X_{1, t-k-1}, X_{2, t-\delta_2} | X_{1, t-k}, X_{2, t-\delta_2}) \nonumber\\
&+\sum_{k=0}^{\delta_2-1} I_L(Y_t; X_{1, t}, X_{2, t-k-1} | X_{1, t}, X_{2, t-k}).
\end{align}
\item[(b)] If $H_L(Y)$ is twice differentiable in $P_Y$, and $$Y_t \xleftrightarrow{\epsilon} (X_{1, t-\tau_1}, X_{2, t-\tau_2}) \xleftrightarrow{\epsilon} (X_{1, t-\tau_1-\mu_1}, X_{2, t-\tau_2-\mu_2})$$ is an $\epsilon$-Markov chain for all $\mu_l, \tau_l\geq 0$, then  $f_2(\delta_1, \delta_2) = O(\epsilon^2) $ and hence (as $\epsilon \rightarrow 0$)
\begin{align}\label{eMarkov1}
    H_L(Y_t |X_{1, t-\delta_1}, X_{2, t-\delta_2})=f_1(\delta_1, \delta_2)+O(\epsilon^2).
\end{align}
\end{itemize}
\end{theorem}
\ifreport
\begin{proof}
See Appendix \ref{Ptheorem1}.
\end{proof}
\else
Due to space limitation, all the proofs are relegated to our technical report \cite{technical_report}.
\fi

\textcolor{blue}{According to Theorem \ref{theorem1}, the minimum training loss $H_L(Y_t|X^m_{t-\delta^m})$ is a function of the age vector $\delta^m$. In addition, $H_L(Y_t|X^m_{t-\delta^m})$ is the difference between two non-decreasing functions of $\delta^m$. Furthermore, the monotonicity of $H_L(Y_t|X^m_{t-\delta^m})$ is characterized by the parameter $\epsilon$ in the $\epsilon$-Markov chain $Y_t \xleftrightarrow{\epsilon} X^m_{t-\tau^m} \xleftrightarrow{\epsilon} X^m_{t-\tau^m-\mu^m}$. If $\epsilon$ is small, then the empirical distribution of training data samples can be accurately approximated as a Markov chain. As a result, the term $f_2(\delta^m)$ is close to zero and $H_L(Y_t|X^m_{t-\delta^m})$ tends to be a non-decreasing function of $\delta^m$. On the other hand, for large $\epsilon$, $H_L(Y_t|X^m_{t-\delta^m})$ is unlikely monotonic in $\delta^m$.}

\textcolor{blue}{As depicted later in Figs. \ref{constantAoI1}-\ref{constantAoI2}, the training loss can indeed be non-monotonic on $\delta^m$. This finding suggests that it is beneficial to investigate non-monotonic age penalty functions, which are more general than the non-decreasing age penalty metrics studied in, e.g., \cite{sun2019,kosta2017age, klugel2019aoi}.}

\subsection{Training Loss under Dynamic AoI}\label{Dynamic_AoI}
In practice, the AoI $\Delta^m(t)$ varies dynamically over time, as shown in Fig. \ref{fig:age1}. Let $P_{\Delta^m}$ denote the empirical distribution of the AoI in the training dataset and $\Delta^m$ be a random vector with the distribution $P_{\Delta^m}$. In the case of dynamic AoI, there are two approaches for training: (i) \emph{Separated training}: the Bayes action for the AoI value $\Delta^m= \delta^m$ is trained by only using the training data samples with AoI value $\delta^m$ \cite{time-series}. The minimum training loss of separate training for $\Delta^m=\delta^m$ is $H_L(Y_t|X^m_{t-\delta^m})$, which has been analyzed in Section \ref{Constant_AoI}. (ii) \emph{Joint training}: the training data samples of all AoI values are combined into a large training dataset and the Bayes actions for different AoI values are trained together. In joint training, the AoI value can  either be included as part of the feature, or be excluded from the feature. If the AoI value is included in the training data (i.e., as part of the feature), the minimum training loss of joint training is $H_L(Y_t| X^m_{t-\Delta^m}, \Delta^m)$. If the AoI value is excluded from the training data, the minimum training loss of joint training is $H_L(Y_t |X^m_{t-\Delta^m})$. Because conditioning reduces the generalized entropy \cite{David_Tse, DawidC}, we have
\begin{align}\label{Dynamic_cod_entropy}
    H_L(Y_t |X^m_{t-\Delta^m}, \Delta^m) \leq H_L(Y_t |X^m_{t-\Delta^m}).
\end{align}
Hence, a smaller training loss can be achieved by including the AoI in the feature. Moreover, similar to \eqref{eq_cond_entropy1}, one can get 
\begin{align}\label{Dynamic_cod_entropy2}
    H_L(Y_t |X^m_{t-\Delta^m}, \Delta^m)=\sum_{\delta^m} P_{\Delta^m}(\delta^m) H_L(Y_t |X^m_{t-\delta^m}).
\end{align}
Therefore, the minimum training loss of joint training is simply the expectation of the training loss of separated training. Our experiment results show that joint training can have a much smaller computational complexity than separated training (see the discussions in Section V-D). Therefore, we suggest to use joint training and add AoI into the feature.

The results in Theorem \ref{theorem1} can be directly generalized to the scenario of dynamic AoI. In particular, if the age processes of two experiments, denoted by subscripts $c$ and $d$, satisfy a sample-path ordering\footnote{ We say $x^m \leq z^m$, if $x_l \leq z_l$ for every $l=1, 2, \ldots, m$. } 
\begin{align}\label{eq_sampleOrdering}
    \Delta_c^m(t) \leq \Delta_d^m(t),~\forall~t,    
\end{align}
then, similar to \eqref{eMarkov1}, one can obtain (as $\epsilon \rightarrow 0$)
\begin{align}\label{dynamicsoln}
    H_L(Y_t|X^m_{t-\Delta^m_c}, \Delta^m_c) \leq H_L(Y_t|X^m_{t-\Delta^m_d}, \Delta^m_d)+O(\epsilon^2).
\end{align}
Next, we show that \eqref{dynamicsoln} can be also proven under a weaker stochastic ordering condition \eqref{eq_stochastic_cond}.
\begin{definition}
\textbf{Univariate Stochastic Ordering:}\cite{stochasticOrder} A random variable $X$ is said to be stochastically smaller than another random variable $Z$, denoted as $X \leq_{st} Z$, if
\begin{align}
    P(X>x) \leq P(Z>x), \ \ \forall x \in \mathbb R.
\end{align}
\end{definition}
\begin{definition}
\textbf{Multivariate Stochastic Ordering:}\cite{stochasticOrder} \textcolor{blue}{A set $U \subseteq \mathbb R^m$ is called upper if $z^m \in U$ whenever $z^m \geq x^m$ and $x^m \in U$.} A random vector $X^m$ is said to be stochastically smaller than another random vector $Z^m$, denoted as $X^m \leq_{st} Z^m$, if
\begin{align}
    P(X^m \in U) \leq P(Z^m \in U), \ \ \ \text{for all upper sets}~ U \subseteq \mathbb R^m.
\end{align}
\end{definition}
\begin{theorem}\label{theorem2}
If $\{(X^m_t, Y_t), t \geq 0\}$ is a stationary random process, $Y_t \xleftrightarrow{\epsilon} X^m_{t-\tau^m} \xleftrightarrow{\epsilon} X^m_{t-\tau^m-\mu^m}$ is an $\epsilon$-Markov chain for all $\mu^m, \tau^m \geq 0$, $H_L(Y)$ is twice differentiable in $P_Y$, and the empirical distributions of the training datasets in two experiments $c$ and $d$ satisfy 
\begin{align}\label{eq_stochastic_cond}
  \Delta^m_{c} \leq_{st} \Delta^m_{d},
\end{align}
then \eqref{dynamicsoln} holds.
\end{theorem}

\ifreport
\begin{proof}
See Appendix \ref{Ptheorem2}.
\end{proof}
\else
\fi

\textcolor{blue}{According to Theorem \ref{theorem2}, if $\Delta_c^m$ is stochastically smaller than $\Delta_d^m$, then the minimum training loss of joint training in Experiment $c$ is approximately smaller than that in Experiment $d$, where the approximation error is of the order $O(\epsilon^2)$.}

\subsection{Testing Loss under Dynamic AoI}\label{Testing_Error}
Let $\mathcal{P^Y}$ denote the space of distributions on $\mathcal{Y}$ and relint($\mathcal{P^Y}$) denote the relative interior of $\mathcal{P^Y}$, i.e., the subset of strictly positive distributions.
\begin{definition}$\beta$-\textbf{neighborhood:}\cite{Lizhong}
Given $\beta \geq 0$, the \emph{$\beta$-neighborhood} of a reference distribution $Q_Y \in relint(\mathcal{P^Y})$ is the set of distributions that are in a Neyman's $\chi^2$-divergence ball of radius $\beta^2$ centered on $Q_Y,$ i.e.,
\begin{align}
\mathcal{N^Y_\beta}(Q_Y) = \left\{ P_Y \in \mathcal{P^Y} : D_{\chi^2}(P_Y || Q_Y) \leq \beta^2\right\}.
\end{align}
\end{definition}
 
\begin{theorem}\label{theorem3}
Let $\{(X^m_t, Y_t), t \geq 0\}$ be a stationary stochastic process.
\begin{itemize}
\item[(a)] If the empirical distributions of training data and testing data are close to each other such that \ignore{$(\tilde X^m_{t-\tilde \Delta^m}, \tilde \Delta^m)$ has the same distribution as $(X^m_{t-\Delta^m}, \Delta^m)$ and}
    \begin{align}\label{T3condition2}
    D_{\chi^2}\left(P_{\tilde Y_t, \tilde X^m_{t-\tilde \Delta^m}, \tilde \Delta^m}||P_{Y_t, X^m_{t-\Delta^m}, \Delta^m}\right) \!\! \leq \beta^2,
    \end{align}
    then the testing loss is close to the minimum training loss, i.e., 
    \begin{align}\label{Eq_Theorem3a}
        H_L(\tilde Y_t; Y_t| \tilde X^m_{t-\tilde \Delta^m}, \tilde \Delta^m)=&H_L(Y_t|X^m_{t-\Delta^m}, \Delta^m)+O(\beta),
    \end{align}
provided that the testing loss is bounded.
\item[(b)] In addition, if $H_L(Y)$ is twice differentiable in $P_Y$, $Y_t \xleftrightarrow{\epsilon} X^m_{t-\tau^m} \xleftrightarrow{\epsilon} X^m_{t-\tau^m-\mu^m}$ is an $\epsilon$-Markov chain for all $\mu^m, \tau^m \geq 0$, and the empirical distributions of the testing datasets in two experiments $c$ and $d$ satisfy $\tilde \Delta^m_{c} \leq_{st} \tilde \Delta^m_{d}$, then the corresponding testing loss satisfies 
    \begin{align}
        H_L(\tilde Y_t; Y_t| \tilde X^m_{t-\tilde \Delta^m_c}, \tilde \Delta^m_c) \leq& H_L(\tilde Y_t; Y_t| \tilde X^m_{t-\tilde \Delta^m_d}, \tilde \Delta^m_d) \nonumber\\
        &+O(\max\{\epsilon^2, \beta\}).
    \end{align}
\end{itemize}
\end{theorem}

\ifreport
\begin{proof}
See Appendix \ref{Ptheorem3}.
\end{proof}
\else
\fi

\textcolor{blue}{As shown in Theorem \ref{theorem3}, if the empirical distributions of training data and testing data are close to each other, then the minimum training loss and testing loss have similar growth patterns as the AoI grows.}

\ignore{\yangc{Need a transition to connect Sec. II with Sec. III.}
\yanc{I have the god feeling that a top-down approach might work better here -- let's first discuss the main findings and then show the detailed proof/analysis.}
\ignore{In section \ref{Constant_AoI}, training loss under a constant AoI, i.e., $\Delta^m(t) = \delta^m$ for all $t$, is studied. It is found that the training loss is a difference of two non-decreasing functions of $\delta^m$. Under some conditions, the training loss becomes close to a non-decreasing function of AoI values. In section \ref{Dynamic_AoI} and \ref{Testing_Error}, for time varying AoI $\Delta^m(t)$, some conditions are provided to show when fresh correlated features improves the training loss and the testing loss respectively.} 

\yanc{it might be better to blend in the metrics into later sections.}}
\begin{figure*}
\minipage[t]{0.32\textwidth}
\includegraphics[width=\columnwidth]{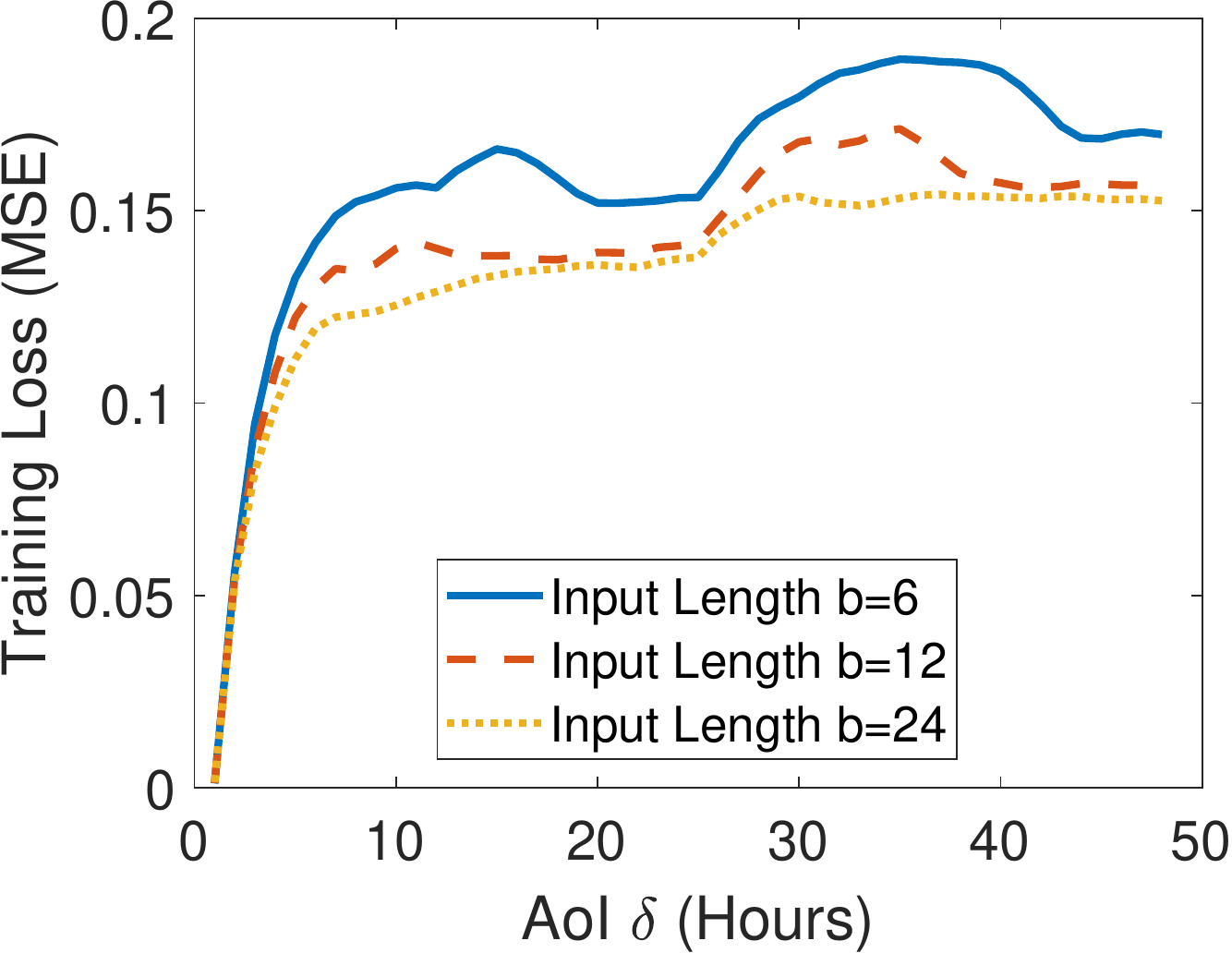}
\caption{Training loss vs. AoI $\delta_1 = \delta_2 = \delta$, where the input length is $b=6$, 12, and 24.}
\label{constantAoI1}
\endminipage\hfill
\minipage[t]{0.32\textwidth}
\includegraphics[width=\textwidth]{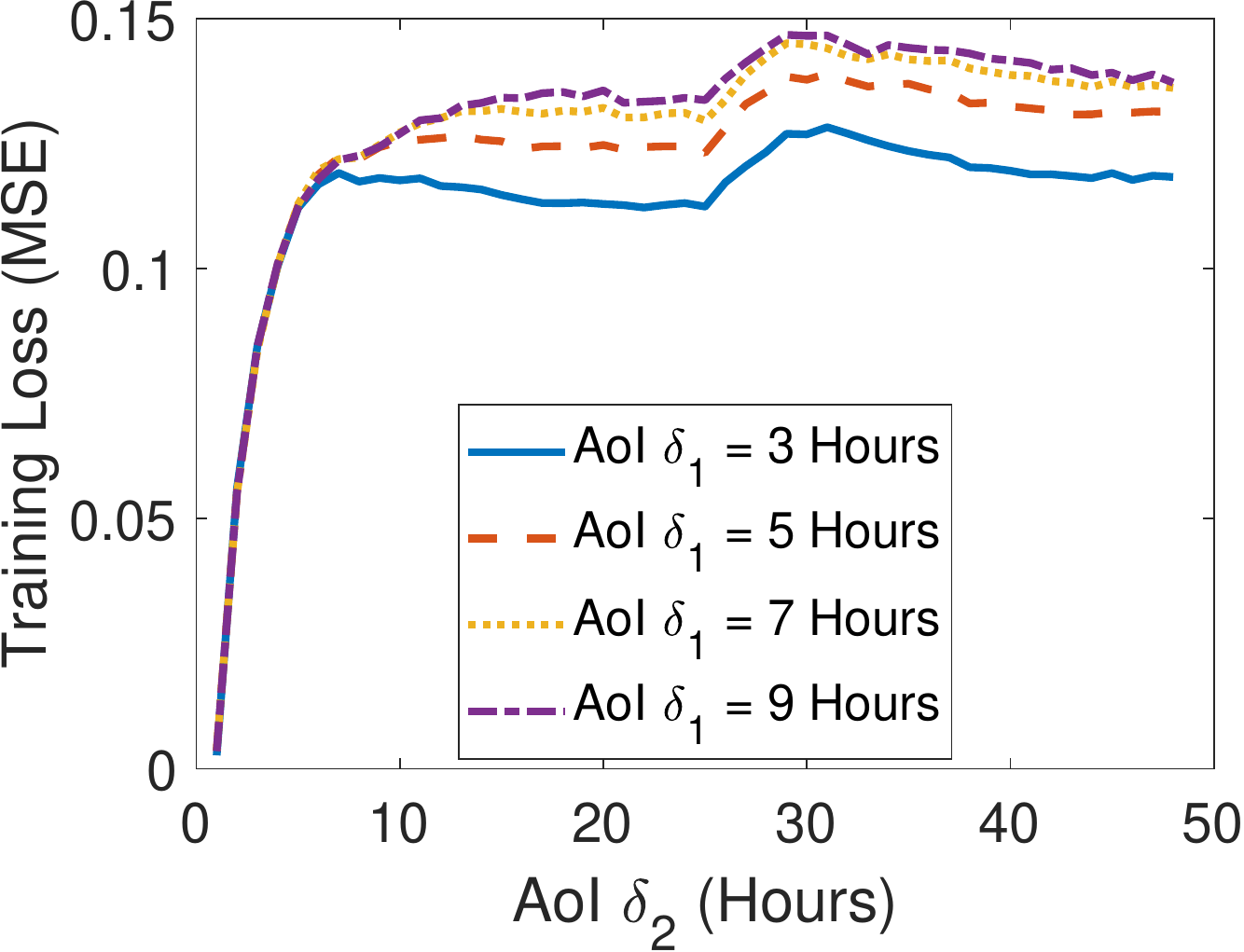}
\caption{Training loss vs. AoI $\delta_2$, where the input length is $b= 24$ and $\delta_1=$3, 5, 7, and 9 hours.}
\label{constantAoI2}
\endminipage\hfill
\minipage[t]{0.32\textwidth}
\includegraphics[width=\columnwidth]{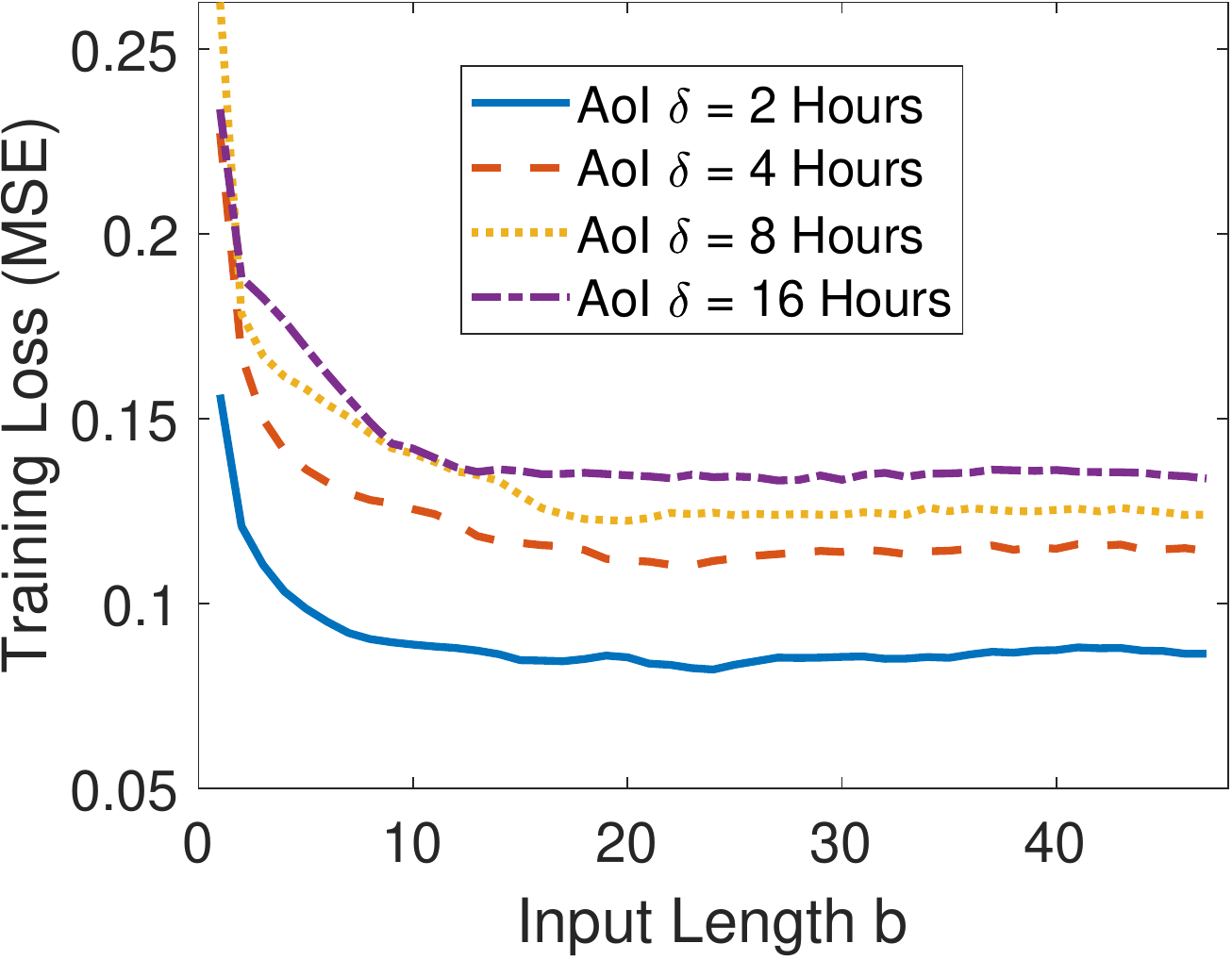}
\caption{Training loss vs. input length $b$, where the AoI is $\delta_1=\delta_2 = $2, 4, 8, and 16 hours.}
\label{fig:input_length}
\endminipage\hfill
\end{figure*}

\begin{figure*}
\minipage[t]{0.64\textwidth}
\includegraphics[width=\textwidth]{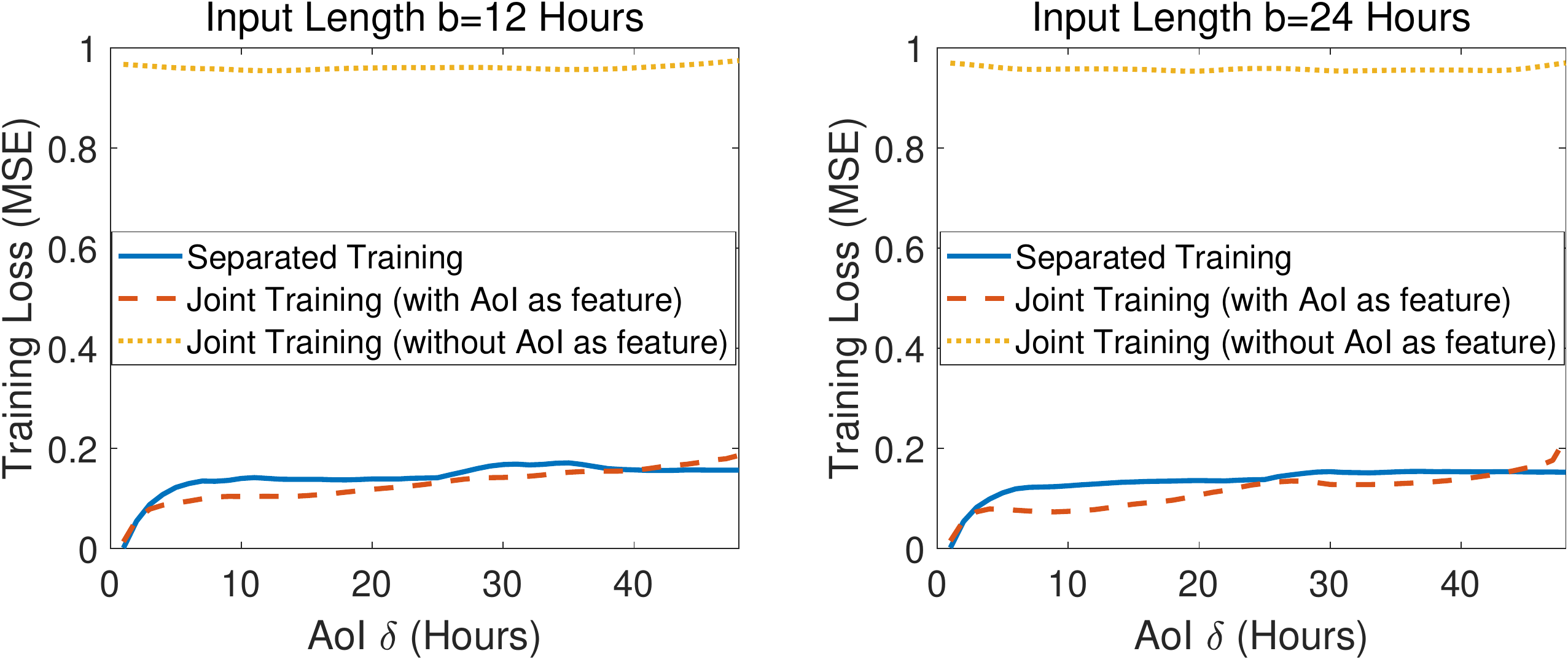}
\caption{Training loss under dynamic AoI $\Delta_1(t)=\Delta_2(t)$, where the input length is $b=12$ and 24. Three training approaches are illustrated: (i) separated training, (ii) joint training with AoI as part of feature, and (iii) joint training without using AoI as part of feature. }
\label{fig:joint_model}
\endminipage\hfill
\minipage[t]{0.32\textwidth}
\includegraphics[width=\columnwidth]{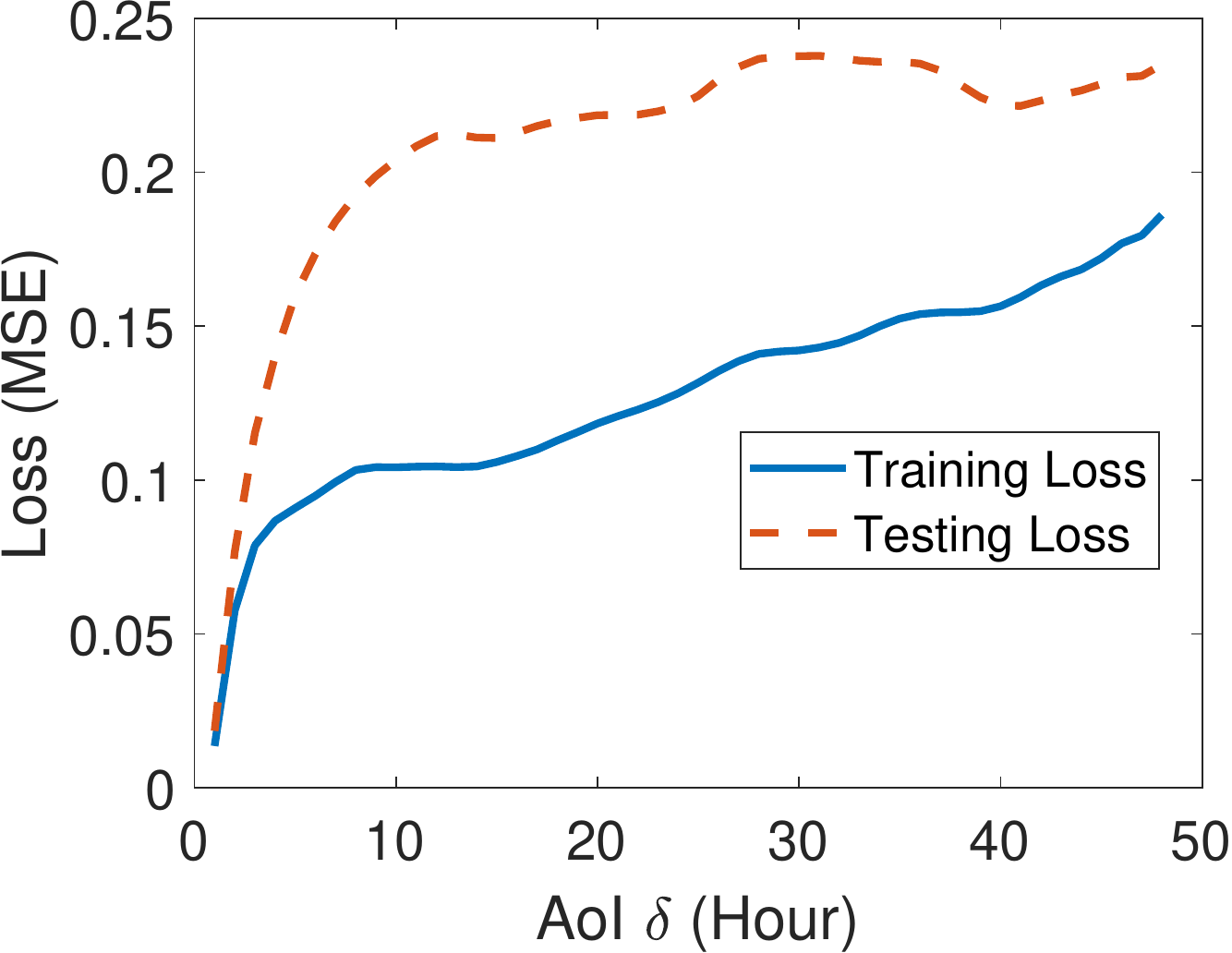}
\caption{Testing loss and training loss vs. dynamic AoI $\Delta_1(t)=\Delta_2(t)$, where the input length is $b=12$. }
\label{fig:test_loss}
\endminipage\hfill
\end{figure*}

\section{Case Studies}
We provide several case studies using a real-world solar power prediction task. A Long Short-Term Memory (LSTM) neural network is used as the prediction model.
Four cases are studied: (i) training loss under constant AoI, (ii) sensitivity analysis of input sequence length $b$ under constant AoI, (iii) training loss under dynamic AoI, and (iv) testing loss under dynamic AoI.

\subsection{Experimental Setup}

\subsubsection{Environment} We use Tensorflow 2\cite{tensorflow2015-whitepaper} on Ubuntu 16.04, and a server with two Intel Xeon E5-2630 v4 CPUs and four GeForce GTX 1080 Ti GPUs to perform the experiments.

\subsubsection{Dataset} A real-world dataset from the probabilistic solar power forecasting task of the Global Energy Forecasting Competition (GEFCom) 2014\cite{HONG2016896} is used for evaluation. The dataset contains 2 years of measurement data for 13 signal processes,
including humidity, thermal, wind speed and direction, and other weather metrics, as explained in \cite{HONG2016896}. Feature $X_{l,t}=(s_{l, t}, \ldots, s_{l,t-b+1})$ of signal $l$ is a time sequence of length $b$. The task is to predict the solar power at an hourly level. We have used two different AoI values $\delta_1$ and $\delta_2$, where 6 features are of the AoI $\delta_1$ and 7 features are of the AoI $\delta_2$. \textcolor{blue}{In jointly training, we use an aggregated dataset with samples of constant AoI $\delta_1=\delta_2$ ranging from 1 to 48 (hours).}
The first year of data is used for training and the second year of data is used for  testing. During preprocessing, both datasets are normalized such that the training dataset has a zero mean and a unity standard derivation. 

\subsubsection{Prediction Model} A Long Short-Term Memory (LSTM) neural network is used for prediction. It composes of one input layer, one hidden layer with 32 LSTM cells, and one fully-connected (dense) output layer. The object of training is to minimize \textcolor{blue}{the expected quadratic loss under the empirical distribution of the training samples}. In other words, empirical mean square error (MSE) $\frac{1}{K}\sum_{i=1}^{K}(y_i-\hat{y}_i)^2$ is minimized, where $K$ is the number of training data entries, $y_i$ is the actual label, and $\hat y_i$ is the predicted label. 
All the experimental results are averaged over multiple runs to reduce the randomness that occurs during the training of LSTM. 
This training setting and the hyper-parameters of Tensorflow 2 training algorithm are consistently used across all evaluations. Note that in theoretical analysis, we have considered that $\Psi$ consists of all possible functions from $\mathcal X^m$ to $\mathcal A$. In practice, neural networks cannot represent such a large function space and the trained weights of the neural network may not be globally optimal. Hence, the training loss \textcolor{blue}{(MSE)} in the experimental study is larger than the minimum training loss analyzed in our theory, but their patterns should be similar.  

\subsection{Training Loss under Constant AoI}
We start with the scenario of separated training, where the AoI is the same across all input features, i.e., $\delta_1=\delta_2=\delta$. As shown in Fig. \ref{constantAoI1}, the training loss is not a monotonic function of AoI $\delta$. \textcolor{blue}{Moreover, with the increase of input sequence length $b$, the training loss becomes close to a non-decreasing function. However, with the increase of input sequence length $b$, the training loss tends to be close to a non-decreasing function. This phenomenon can be interpreted by using Shannon's high-order Markov model for information sources\cite{shannon}. Specifically, the feature process $X_{l,t}$ can be approximated as a Markov chain model with order $b$, where the approximation error reduces as the order $b$ grows. Because of this, the training loss is far away from an increasing function of the AoI when $b$ = 6 or 12, and is nearly an increasing function of the AoI when $b$ = 24.} 

Next, we consider a more general case where $\delta_1$ and $\delta_2$ can have different values. The results are shown in Fig. \ref{constantAoI2}, where both $\delta_1$ and $\delta_2$ affect the training loss in accordance to Theorem \ref{theorem1}.
\ignore{Next, we consider that input features have two different AoI values $\delta_1$ and $\delta_2$ , which is a more practical case as we may experience heterogeneous staleness from various sources including communication, processing, preemptive resources allocation, etc. The results are shown in Fig. \ref{constantAoI2}, where both $\delta_1$ and $\delta_2$ contribute to the increase of training loss as discussed in \ref{Constant_AoI}.}

\subsection{Sensitivity Analysis of Input Sequence Length $b$ Under Constant AoI}
With the increase of input length sequence $b$, the training loss is expected to decrease, because conditioning reduces the generalized conditional entropy.
To show this, we evaluate the training loss for a wide range of input lengths, as shown in Fig. \ref{fig:input_length}. The observations are two-folded. First, for a given AoI value, the training loss is a non-increasing function of the input length $b$. Second, even with the increase of input length, a larger AoI leads to a larger training loss. The observed pattern agrees with our theoretical analysis.

\subsection{Training Loss under Dynamic AoI}
In the \emph{separated training} method considered above, one predictor (i.e., an LSTM neural network) is trained for every AoI value. Hence, a number of predictors are needed. Training these predictors may incur a huge computational cost, which would be impractical for prediction tasks with huge datasets. As discussed in Section \ref{Dynamic_AoI}, \emph{joint training} is a better approach, where a single predictor is trained by jointly using the input samples of different AoI values. As plotted in Fig. \ref{fig:joint_model}, if the AoI is excluded from the feature, \textcolor{blue}{joint training has a significant performance degradation compared to separated training as described in \eqref{Dynamic_cod_entropy}. However, with AoI as a part of the input feature, joint training has comparable performance as separated training, which agrees with the relationship in \eqref{Dynamic_cod_entropy2}. For a wide range of AoI values, the performance of joint training is slightly better than separated training because it uses all training data where separated predictor can only be trained on the data with certain AoI. This phenomenon is in alignment with data augmentation \cite{wong2016understanding}. In current training dataset for jointly training, we mainly focus on small AoIs so the joint model may not be as good as separated training for very large AoI values.} 
But such long AoI is rare in real-world applications. Thus, the idea of appending AoI to the input features is a good idea for time-varying AoI.

\subsection{Testing Loss under Dynamic AoI}
Training loss and testing loss  are compared in Fig. \ref{fig:test_loss} for joint training under dynamic AoI. One can observe that the testing loss is not monotonic in AoI, but it has a growing trend that is similar to the training loss. In addition, the testing loss is larger than the training loss, which is caused by the difference between the empirical distributions of training and testing datasets. Such a difference is quite normal in machine learning, which occurs, for instance, if the datasets for training/testing are not sufficiently large or if there is concept shift \cite{concept_shift}. 

\ignore{First, we validate theoretical results from Section \ref{Testing_Error} that the gap between training and testing performance should be close given similar distributions. Then we further validate that the testing performance of the prediction model would be under similar impact as the training performance from different AoI values. As shown in Fig. \ref{fig:test_loss}, the testing loss under different input lengths is close to their testing loss given small AoI values, suggesting that similar distributions do yield good testing loss. For larger AoI, there is a gap between training and testing performance, which is expected as the the training data are biased with staleness thus the model suffers more from concept shift. As the AoI grows, the testing loss grows with fluctuations, which is in align with the pattern we observed from the testing loss and our theoretical analysis.}

\ignore{\yanc{may need a summary and future work section.}}

\section{Conclusion}
In this paper, we have presented a unified theoretical framework to analyze how the age of correlated features affects the performance in supervised learning based forecasting. It has been shown that the minimum training loss is a function of the age, which is not necessarily monotonic. Conditions have been provided under which the training loss and testing loss are approximately non-decreasing in age. Our investigation suggests that, by (i) jointly training the forecasting actions for different age values and (ii) adding the age value into the input feature, both forecasting accuracy and computational complexity can be greatly improved.

\bibliographystyle{IEEEtran}
\bibliography{refshisher1,experiments_refs}
\appendices
\section{Proof of Theorem \ref{theorem1}}\label{Ptheorem1}
Part (a): By the definitions of generalized conditional entropy and generalized mutual information, one can obtain 
\begin{align}
&H_L(Y|X_1, X_2, Z_1, Z_2) \nonumber\\
=&H_L(Y|Z_1, Z_2) - I_L(Y; X_1, X_2|Z_1, Z_2)\nonumber\\
=&H_L(Y|X_1, X_2) - I_L(Y; Z_1, Z_2|X_1, X_2),
\end{align}
which yields
\begin{align}\label{PThm}
H_L(Y|Z_1, Z_2)=H_L(Y|X_1, X_2)+&I_L(Y; X_1, X_2|Z_1, Z_2)\nonumber\\
-& I_L(Y; Z_1, Z_2|X_1, X_2).
\end{align}
Now, if we replace $Z_1, Z_2, X_1, X_2$, and $Y$ in \eqref{PThm} with $X_{1, t-k-1},  X_{2, t-\delta_2}, X_{1, t-k}, X_{2, t-\delta_2},$ and $Y_t$, respectively, then \eqref{PThm} becomes 
\begin{align}\label{PThm1}
&H_L(Y_t|X_{1, t-k-1},X_{2, t-\delta_2})\nonumber\\
=&H_L(Y_t| X_{1, t-k}, X_{2, t-\delta_2}) \nonumber\\
&+I_L(Y_t; X_{1, t-k}, X_{2, t-\delta_2}|X_{1, t-k-1},  X_{2, t-\delta_2})\nonumber\\
&- I_L(Y_t; X_{1, t-k-1}, X_{2, t-\delta_2}|X_{1, t-k},X_{2, t-\delta_2}).
\end{align}
Equation \eqref{PThm1} is valid for any value of $t$ and $k$. Therefore, taking summation of $H_L(Y_t|X_{1, t-k-1},X_{2, t-\delta_2})$ from $k=0$ to $\delta_1-1$, we get
\begin{align}\label{PThm2}
&H_L(Y_t|X_{1, t-\delta_1},X_{2, t-\delta_2})\nonumber\\
=&H_L(Y_t| X_{1, t}, X_{2, t-\delta_2}) \nonumber\\
&+\sum_{k=0}^{\delta_1-1}I_L(Y_t; X_{1, t-k}, X_{2, t-\delta_2}|X_{1, t-k-1},  X_{2, t-\delta_2})\nonumber\\
&-\sum_{k=0}^{\delta_1-1} I_L(Y_t; X_{1, t-k-1}, X_{2, t-\delta_2}|X_{1, t-k},X_{2, t-\delta_2}).
\end{align}
Similarly, we can establish that
\begin{align}\label{PThm3}
&H_L(Y_t|X_{1, t},X_{2, t-\delta_2})\nonumber\\
=&H_L(Y_t| X_{1, t}, X_{2, t}) \nonumber\\
&+\sum_{k=0}^{\delta_2-1}I_L(Y_t; X_{1, t}, X_{2, t-k}|X_{1, t},  X_{2, t-k-1})\nonumber\\
&-\sum_{k=0}^{\delta_2-1} I_L(Y_t; X_{1, t}, X_{2, t-k-1}|X_{1, t},X_{2, t-k}).
\end{align}
Combining \eqref{PThm2} and \eqref{PThm3}, we get \eqref{eMarkov}. 

Because, mutual information is a non-negative term, one can observe from \eqref{functionf_1} that for a fixed $\delta_2$, the functions $f_1(\delta_1, \delta_2)$ is a non-decreasing function of $\delta_1$. Moreover, $f_1(\delta_1, \delta_2)$ can be written in another form as 
\begin{align}\label{functionf_12}
&f_1(\delta_1, \delta_2)\! \nonumber\\
=&H_L(Y_t| X_{1,t}, X_{2,t}) \nonumber\\
& +\sum_{k=0}^{\delta_2-1} I_L(Y_t; X_{1, t-\delta_1}, X_{2, t-k} | X_{1, t-\delta_1}, X_{2, t-k-1}) \nonumber\\
& +\sum_{k=0}^{\delta_1-1} I_L(Y_t; X_{1, t-k}, X_{2, t} | X_{1, t-k-1}, X_{2, t}).
\end{align}
From the above equation, it is also observed that for a fixed $\delta_1$, the functions $f_1(\delta_1, \delta_2)$ is a non-decreasing function of $\delta_2$. Therefore, \eqref{functionf_1} and \eqref{functionf_12} imply that $f_1(\delta_1, \delta_2)$ is a non-decreasing function of $\delta_1$ and $\delta_2$. Similarly, from \eqref{functionf_2}, we can deduce that $f_2(\delta_1, \delta_2)$ is a non-decreasing function of $\delta_1$ and $\delta_2$.

Part (b): 
Because  $H_L(Y)$ is twice differentiable in $P_Y$ and $Y_t \xleftrightarrow{\epsilon} (X_{1, t-\tau_1}, X_{2, t-\tau_2}) \xleftrightarrow{\epsilon} (X_{1, t-\tau_1-\mu_1}, X_{2, t-\tau_2-\mu_2})$ is an $\epsilon$-Markov chain for all $\mu_l, \tau_l\geq 0$, by using Lemma \ref{Lemma_CMI}, $f_2(\delta_1, \delta_2)$ satisfies
\begin{align}
    f_2(\delta_1, \delta_2)&=\sum_{k=0}^{\delta_1-1} O(\epsilon^2)+\sum_{k=0}^{\delta_2-1} O(\epsilon^2)\nonumber\\
    &=O(\epsilon^2).
\end{align}
This concludes the proof.

\section{Proof of Theorem \ref{theorem2}}\label{Ptheorem2}
By using Theorem 1 and substituting Equation \eqref{eMarkov1} into \eqref{Dynamic_cod_entropy2}, we obtain
\begin{align}\label{PThm2a}
H_L(Y_t |X^m_{t-\Delta^m}, \Delta^m)=\mathbb E_{\Delta^m \sim P_{\Delta^m}}[f_1(\Delta^m)]+O(\epsilon^2).
\end{align}
The expectations in the above equation exist as entropy and conditional entropy are considered bounded and $H_L(Y_t |X^m_{t-\Delta^m}, \Delta^m) \leq H_L(Y_t)< \infty.$ Moreover, $f_1( \delta^m)$ is a non-decreasing function of $\delta^m$. Therefore, as $ \Delta^m_{c} \leq_{st} \Delta^m_{d}$, we get \cite{stochasticOrder}
\begin{align}\label{PThm2b}
    \mathbb E_{\Delta^m_c \sim P_{\Delta^m_c}}[f_1(\Delta^m_c)] \leq   \mathbb E_{\Delta^m_d \sim P_{\Delta^m_d}}[f_1( \Delta^m_d)].
\end{align}
By using \eqref{PThm2a} and \eqref{PThm2b}, we obtain
\begin{align}\label{PThm2c}
H_L(Y_t |X^m_{t-\Delta^m_c}, \Delta^m_c) \leq H_L(Y_t |X^m_{t-\Delta^m_d}, \Delta^m_d)+O(\epsilon^2).
\end{align}
This concludes the proof.

\section{Proof of Theorem \ref{theorem3}}\label{Ptheorem3}
Part (a): First, we provide the following lemma.

\begin{lemma}\label{lemma2}
 If 
    \begin{align}\label{condition2}
        D_{\chi^2}\left(P_{\tilde Y,\tilde X}|| P_{Y,X}\right) \leq \beta^2,
    \end{align}
    then 
      \begin{align}\label{result12}
        P_{\tilde X}(x)=P_{X}(x)+O(\beta), \ \forall x \in \mathcal X,
    \end{align}
    \begin{align}\label{result2}
        P_{\tilde Y|\tilde X}(y|x)=P_{Y|X}(y|x)+O(\beta), \ \forall y \in \mathcal Y, \forall x \in \mathcal X.
    \end{align}
In addition, if $H_L(\tilde Y; Y|\tilde X)$ is bounded, then
    \begin{align}\label{result3}
         H_L(\tilde Y; Y|\tilde X)=H_L(Y|X)+O(\beta).
    \end{align}
    \end{lemma}
\begin{proof}
See Appendix \ref{Plemma2}.
\end{proof}
Using Lemma \ref{lemma2}, we prove part (a) of Theorem \ref{theorem3}. Because the condition \eqref{T3condition2} holds and the testing loss is assumed bounded, if we replace $X,~Y,~\tilde X$, and $\tilde Y$ in \eqref{result3} with $(X^m_{t-\Delta^m},~ \Delta^m),~ Y_t, ~(\tilde X^m_{t-\tilde \Delta^m}, \tilde \Delta^m)$, and $\tilde Y_t$, then \eqref{result3} becomes \eqref{Eq_Theorem3a}.

Part (b): By using Theorem \ref{theorem2} and \eqref{Eq_Theorem3a}, we get
 \begin{align}
        H_L(\tilde Y_t; Y_t| \tilde X^m_{t-\tilde \Delta^m_a}, \tilde \Delta^m_a) \leq& H_L(\tilde Y_t; Y_t| \tilde X^m_{t-\tilde \Delta^m_b}, \tilde \Delta^m_b) \nonumber\\
        &+O(\beta)+O(\epsilon^2),
    \end{align}
 where $O(\beta)+O(\epsilon^2)=O(\max\{\epsilon^2, \beta\})$. This concludes the proof.
 
 \section{Proof of Lemma \ref{lemma2}}\label{Plemma2}
Because 
\begin{align}
    D_{\chi^2}(P_{\tilde Y, \tilde X} || P_{Y,X}) \geq D_{\chi^2}(P_{\tilde X} || P_{X}),
\end{align}
from \eqref{condition2}, we have
\begin{align}\label{condition1}
    D_{\chi^2}(P_{\tilde X} || P_{X}) \leq \beta^2.
\end{align}
 By using the definition of Neyman's $\chi^2$-divergence, from inequality \eqref{condition1}, we can show \eqref{result12}. Now, from inequality \eqref{condition2}, we get
\begin{align}\label{recondition2}
    \sum_{\substack{x \in \mathcal X \\ y \in \mathcal Y}} \frac{\left[P_{\tilde Y|\tilde X}(y|x)P_{\tilde X}(x)-P_{Y| X}(y|x)P_{X}(x)\right]^2}{P_{Y| X}(y|x)P_{X}(x)} \leq \beta^2.
\end{align}
Substituting \eqref{result12} into the above inequality and after some algebraic operations, we have
\begin{align}
    \sum_{\substack{x \in \mathcal X \\ y \in \mathcal Y}} P_X(x) \frac{\left[P_{\tilde Y|\tilde X}(y|x)-P_{Y| X}(y|x)\right]^2}{P_{Y| X}(y|x)} \leq O(\beta^2).
\end{align}
From \eqref{recondition2}, it is simple to observe that $P_{Y|X}(y|x) \neq 0$ and $P_{X}(x) \neq 0$ for all $x \in \mathcal X$ and $y \in \mathcal Y$. Because $P_X(x) \neq 0,~ \forall x \in \mathcal X$, we can further obtain
\begin{align}\label{Thm3a}
     \sum_{y \in \mathcal Y} \frac{\left[P_{\tilde Y|\tilde X}(y|x)-P_{Y| X}(y|x)\right]^2}{P_{Y| X}(y|x)} \leq \dfrac{O(\beta^2)}{P_{X}(x)}, \forall x \in \mathcal X.
\end{align}
An equivalent representation of the above inequality is that there exists a function $C: \mathcal Y \mapsto \mathbb R$ such that for all $x \in \mathcal X$ and $y \in \mathcal Y$,
\begin{align}\label{Thm3b}
    P_{\tilde Y| \tilde X}(y|x)-P_{Y| X}(y|x)=&\dfrac{O(\beta)}{\sqrt{P_{X}(x)}} C(y) \sqrt{P_{Y| X}(y|x)} \nonumber\\
    =& O(\beta), 
\end{align}
where 
\begin{align}
    \sum_{y \in \mathcal Y} C^2(y) \leq 1.
\end{align}
Equation \eqref{Thm3b} implies \eqref{result2}.

Define a function $g(P_{\tilde Y|\tilde X=x})$ as 
\begin{align}\label{Thm3c}
    g(P_{\tilde Y|\tilde X=x})=&H_L(\tilde Y; Y|\tilde X=x)-H_L(Y|X=x) \nonumber\\
    =&\mathbb E_{\tilde Y \sim P_{\tilde Y|\tilde X=x}}[L(Y, a_{P_{Y|X=x}})]-H_L(Y|X=x)].
\end{align}
In this lemma, we assume that conditional cross entropy $H_L(\tilde Y; Y|\tilde X)$ is bounded. Also, conditional entropy $H_L(Y|X)$ is considered bounded in this paper. Thus, $g(P_{\tilde Y|\tilde X=x})$ is bounded for all $x \in \mathcal X$.
By the definition of conditional entropy and conditional cross entropy, we obtain
\begin{align}\label{Thm3d}
    &g(P_{\tilde Y|\tilde X=x})\nonumber\\
    =&\sum_{y\in \mathcal Y} \left(P_{\tilde Y|\tilde X}(y|x)-P_{Y|X}(y|x)\right)L(y, a_{P_{Y|X=x}})\nonumber\\
    =&O(\beta).
\end{align}
The last equality is obtained by substituting the value of $P_{\tilde Y|\tilde X}(y|x)$ from \eqref{result2}. Substituting \eqref{Thm3c} into \eqref{Thm3d}, we get 
\begin{align}\label{Thm3e}
    H_L(\tilde Y; Y|\tilde X=x)=H_L(Y|X=x)+O(\beta).
\end{align}
By using \eqref{result12} and \eqref{Thm3e}, yields
\begin{align}\label{Thm3g}
  H_L(\tilde Y; Y|\tilde X)=& \sum_{x \in \mathcal X} P_{\tilde X}(x)H_L(\tilde Y; Y|\tilde X=x)\nonumber\\
  =&\sum_{x \in \mathcal X} \left(P_{X}(x)+O(\beta)\right)H_L(\tilde Y; Y|\tilde X=x)\nonumber\\
  =&\sum_{x \in \mathcal X} P_{ X}(x)H_L(\tilde Y; Y|\tilde X=x)+O(\beta)\nonumber\\
  =&\sum_{x \in \mathcal{X}}P_{ X}(x) \left(H_L(Y|X=x)+O(\beta)\right)+O(\beta)\nonumber\\
  =&\sum_{x \in \mathcal{X}}P_{ X}(x)H_L(Y|X=x)+O(\beta)\nonumber\\
  =&H_L(Y|X)+O(\beta).
\end{align}
This concludes the proof.
\end{document}